\documentclass[12pt,dvips,generic,noinfoline]{imsart}

\RequirePackage[OT1]{fontenc}
\RequirePackage{amsthm,amsmath,amsfonts,natbib}
\RequirePackage{hypernat}
\usepackage[ruled,section]{algorithm}
\usepackage[noend]{algorithmic}
\usepackage{fullpage}
\usepackage{graphicx}
\usepackage{pstricks}

\DeclareMathOperator{\rank}{rank}
\DeclareMathOperator{\diag}{diag}

\newcommand{\EE}[1]{\mathbb{E}({#1})}
\def\interleave{|\kern-.25ex|\kern-.25ex|}
\def\interleavesub{|\kern-.15ex|\kern-.15ex|}
\newcommand{\nnorm}[1]{\interleave {#1} \interleave}
\newcommand{\nnormsub}[1]{\interleavesub {#1} \interleavesub}
\newcommand{\nNorm}[1]{\left|\kern-.25ex\left|\kern-.25ex\left| {#1}\right|\kern-.25ex\right|\kern-.25ex\right|}
\newcommand{\norm}[1]{\left\|{#1}\right\|}
\newcommand{\trnorm}[1]{\norm{{#1}}_{*}}
\newcommand{\spnorm}[1]{\norm{{#1}}_{\mathrm{sp}}}

\newcommand{\trnnorm}[1]{\nnorm{{#1}}_{*}}
\newcommand{\spnnorm}[1]{\nnorm{{#1}}_{\mathrm{sp}}}
\newcommand{\spnnormsub}[1]{\nnormsub{{#1}}_{\mathrm{sp}}}

\newcommand{\Norm}[1]{\left\|{#1}\right\|}
\newcommand{\Trnorm}[1]{\Norm{{#1}}_{*}}
\newcommand{\Spnorm}[1]{\Norm{{#1}}_{\mathrm{sp}}}

\def\myllangle{\langle\mspace{-4mu}\langle\mspace{1mu}}
\def\myrrangle{\mspace{1mu}\rangle\mspace{-4mu}\rangle}

\numberwithin{equation}{section}
\theoremstyle{plain}

\newtheorem{proposition}{Proposition}[section]

\newtheoremstyle{remark}{\topsep}{\topsep}%
     {\normalfont}% Body font
     {}           % Indent amount (empty = no indent, \parindent = para indent)
     {\bfseries}  % Thm head font
     {.}          % Punctuation after thm head
     {.5em}       % Space after thm head (\newline = linebreak)
     {\thmname{#1}\thmnumber{ #2}\thmnote{ #3}}% Thm head spec
\theoremstyle{remark}

\def\comma{\unskip,~}

\long\def\comment#1{}
\def\reals{{\mathbb R}}

\def\E{{\mathbb E}}
\def\supp{\mathop{\text{supp}\kern.2ex}}

\let\hat\widehat
\let\tilde\widetilde

\let\hat\widehat
\let\tilde\widetilde
\def\given{{\,|\,}}

\def\1{{(1)}}
\def\2{{(2)}}

\def\M{{\mathcal{M}}}
\def\tr{\mathop{\text{tr}}}
\long\def\comment#1{}
\def\comma{\unskip,~}

\long\def\comment#1{}
\def\reals{{\mathbb R}}

\def\E{{\mathbb E}}
\def\Cov{\mathop{\text{Cov}}}
\def\supp{\mathop{\text{supp}\kern.2ex}}

\let\tilde\widetilde

\let\hat\widehat
\let\tilde\widetilde
\def\given{{\,|\,}}

\def\1{{(1)}}
\def\2{{(2)}}

\def\tr{\mathop{\text{tr}}}
\long\def\comment#1{}

\def\cram{{\sc cram}}
\def\spam{{\small\sc SpAM}}
\def\diag{\mathop{\rm diag}}
\def\ones{\mathbf{1}}
\def\threebars{\mbox{$|\kern-.25ex|\kern-.25ex|$}}

\def\rank{\mathop{\rm rank}}
\def\S{\mathcal{S}}
\def\H{\mathcal{H}}
\def\K{{K}}
\def\rank{\mathop{\rm rank}}
\def\half{{1/2}}

\def\M{\mathbb{M}}
\def\F{\mathbb{F}}
\def\pinv{{-1}}
\def\Res{Z}
\def\Proj{P}
\def\cN{{\mathcal N}}
\def\cT{{\mathcal H}}

\begin{document}

\begin{frontmatter}
\title{Nonparametric Reduced Rank Regression}%\protect\thanksref{T1}}
\runtitle{Nonparametric Reduced Rank Regression}

\begin{aug}
\vskip10pt
\author{\fnms{Rina} \snm{Foygel${}^{\dag,*}$}\ead[label=e1]{rinafb@stanford.edu}}
\comma 
\author{\fnms{Michael} \snm{Horrell${}^{\dag}$}\ead[label=e2]{horrell@galton.uchicago.edu}}
\comma
\author{\fnms{Mathias} \snm{Drton${}^{\dag,\ddag}$}\ead[label=e3]{md5@uw.edu}}
\and
\author{\fnms{John} \snm{Lafferty${}^{\dag\ss}$}\ead[label=e4]{lafferty@galton.uchicago.edu}}
\address{
\vskip1pt
\begin{tabular}{ccc}
${}^{*}$Department of Statistics & ${}^{\dag}$Department of Statistics & ${}^{\ddag}$Department of Statistics\\
Stanford University & ${}^{\ss}$Department of Computer Science & University of Washington\\
& The University of Chicago & 
\end{tabular}
%\address{The University of Chicago\\Department of Statistics}
\\[10pt]
\today\\[5pt]
%\printead{e1,e2,e3}}
}
\end{aug}

%\begin{aug}
%\author{
%Rina Foygel,\; Michael Horrell,\; Mathias Drton$,\; John Lafferty\\[5pt]
%\begin{tabular}{ccc}
%$Department of Statistics & Department of Statistics & Department of Statistics \\
%Stanford University & University of Chicago & University of
%Washington\\
%\end{tabular}
%}

%\address{Carnegie Mellon University\\[0pt]
%\today\\[5pt]
%%\printead{e1,e2,e3}}
%}
%\end{aug}

\begin{abstract}
  We propose an approach to multivariate nonparametric regression that
  generalizes reduced rank regression for linear models.  An additive
  model is estimated for each dimension of a $q$-dimensional response,
  with a shared $p$-dimensional predictor variable.  To control the
  complexity of the model, we employ a functional form of the Ky-Fan
  or nuclear norm, resulting in a set of function estimates that have
  low rank.  Backfitting algorithms are derived and justified using a
  nonparametric form of the nuclear norm subdifferential.  Oracle
  inequalities on excess risk are derived that exhibit the scaling
  behavior of the procedure in the high dimensional setting.  The
  methods are illustrated on gene expression data.
\end{abstract}


%\tableofcontents
\vskip20pt
\end{frontmatter}

\maketitle

\begin{abstract}
  We propose an approach to multivariate nonparametric regression that
  generalizes reduced rank regression for linear models.  An additive
  model is estimated for each dimension of a $q$-dimensional response,
  with a shared $p$-dimensional predictor variable.  To control the
  complexity of the model, we employ a functional form of the Ky-Fan
  or nuclear norm, resulting in a set of function estimates that have
  low rank.  Backfitting algorithms are derived and justified using a
  nonparametric form of the nuclear norm subdifferential.  Oracle
  inequalities on excess risk are derived that exhibit the scaling
  behavior of the procedure in the high dimensional setting.  The
  methods are illustrated on gene expression data.
\end{abstract}
\makeatletter
\def\blfootnote{\gdef\@thefnmark{}\@footnotetext}
\blfootnote{This is an extended version of a paper
presented at NIPS \citep{NRRR:NIPS}.}

%\clearpage
\vskip20pt
\section{Introduction}

In the multivariate regression problem the objective is to estimate
the conditional mean 
\[
\E(Y\given X) = m(X) = (m^1(X),\ldots,
m^q(X))^\top,
\]
where $Y$ is a $q$-dimensional response vector, $X$ is a
$p$-dimensional covariate vector, and we are given a sample of $n$
i.i.d.~pairs $\{(X^{(i)}, Y^{(i)})\}$ from the joint distribution of
$X$ and $Y$.  This is also referred to as multi-task learning in the
machine learning literature.  Under a linear model, the mean is
estimated as $m(X) = BX$ where $B\in\reals^{q\times p}$ is a $q\times
p$ matrix of regression coefficients.  When the dimensions $p$ and $q$
are large relative to the sample size $n$, the coefficients of $B$
cannot be reliably estimated, without further assumptions.

In reduced rank regression the matrix $B$ is estimated under a rank
constraint $r = \rank(B) \leq C$, so that the rows or columns of $B$
lie in an $r$-dimensional subspace of $\reals^q$ or $\reals^p$.
Intuitively, this implies that the model is based on a smaller number
of features than the ambient dimensionality $p$ would suggest, or that
the tasks representing the components $Y^k$ of the response are
closely related.  In low dimensions, the constrained rank model can be
computed as an orthogonal projection of the least squares solution; but
in high dimensions this is not well defined.

Recent research has studied the use of the nuclear norm as a convex
surrogate for the rank constraint.  The nuclear norm $\|B\|_*$, also
known as the trace or Ky-Fan norm, is the sum of the singular vectors
of $B$.  A rank constraint can be thought of as imposing sparsity, but
in an unknown basis; the nuclear norm plays the role of the $\ell_1$
norm in sparse estimation.  Its use for low rank estimation problems
was proposed by \cite{fazel:02}.  More recently, nuclear norm
regularization in multivariate linear regression has been studied by
\cite{Yuan:07}, and by \cite{Negahban:11}, who analyzed the scaling properties of
the procedure in high dimensions.

In this paper we study nonparametric parallels of reduced rank linear
models. We focus our attention on additive models, so that the
regression function $m(X) = (m^1(X), \ldots, m^q(X))^\top$ has each
component $m^k(X) = \sum_{j=1}^p m^k_j(X_j)$ equal to a sum of $p$
functions, one for each covariate.  The objective is then to estimate
the $q\times p$ matrix of functions $M(X) = \left[m^k_j(X_j)\right]$.

%% so that the
%% regression function $m(X) = \E(Y\given X)$ is assumed to be of the
%% form $ m(X) = (m^1(X), \ldots, m^q(X))^\top$ where each component
%% $m^k(X) = \sum_{j=1}^p m^k_j(X_j)$ is a sum of $p$ functions, one for
%% each covariate.  The objective is then to estimate the $q\times p$
%% matrix of functions $M(X) = \left[m^k_j(X_j)\right]$.

The first problem we address, in Section~\ref{sec:penalties}, is to
determine a replacement for the regularization penalty $\|B\|_*$ in
the linear model.  Because we must estimate a matrix of functions, the
analogue of the nuclear norm is not immediately apparent.  We propose
two related regularization penalties for nonparametric low rank
regression, and show how they specialize to the linear case.  We then
study, in Section~\ref{sec:subdiff}, the (infinite dimensional)
subdifferential of these penalties.  In the population setting, this
leads to stationary conditions for the minimizer of the regularized
mean squared error.  This subdifferential calculus then justifies
penalized backfitting algorithms for carrying out the optimization for
a finite sample.  Constrained rank additive models (\cram) for
multivariate regression are analogous to sparse additive models
(\spam) for the case where the response is 1-dimensional
\citep{Ravikumar:09} (studied also in the reproducing kernel Hilbert space setting by \cite{Raskutti:10}),
 but with the goal of recovering a low-rank matrix rather
than an entry-wise sparse vector.
  The backfitting algorithms we derive in
Section~\ref{sec:stationary} are analogous to the iterative smoothing
and soft thresholding backfitting algorithms for \spam{} proposed by
\cite{Ravikumar:09}.  
%However, the regularization penalties we propose
%are not simple generalizations of the \spam{} penalties for $q=1$.  
A uniform bound on the excess risk of the estimator relative to an
oracle is given Section~\ref{sec:excess}.  This shows the statistical
scaling behavior of the methods for prediction.  The analysis requires
a concentration result for nonparametric covariance matrices in the
spectral norm.  Experiments with synthetic, gene, and biochemistry data are given in
Section~\ref{sec:genes}, which are used to illustrate different
facets of the proposed nonparametric
reduced rank regression techniques.

\section{Nonparametric Nuclear Norm Penalization}
\label{sec:penalties}

We begin by presenting the penalty that we will use to induce
nonparametric regression estimates to be low rank.  To motivate our
choice of penalty and provide some intuition, suppose that
$f^1(x),\ldots, f^q(x)$ are $q$ smooth one-dimensional functions with
a common domain. What does it mean for this collection of functions to
be low rank?  Let $x^{(1)}, x^{(2)}, \ldots, x^{(n)}$ be a collection
of points in the common domain of the functions.  We require that the
$n\times q$ matrix of function values $\F(x^{(1:n)}) = [f^k(x^{(i)})]$ is
low rank.  This matrix is of rank at most $r < q$ for every set
$\{x^{(i)}\}$ of arbitrary size $n$ if and only if the functions $\{f^k\}$
are $r$-linearly independent---each function can be written as a
linear combination of $r$ of the other functions.

In the multivariate regression setting, but still assuming the domain
is one-dimensional for simplicity ($q>1$ and $p=1$), we have a random
sample $X^{(1)}, \ldots, X^{(n)}$.  Consider the $n\times q$ sample
matrix $\M = [m^k(X^{(i)})]$ associated with a vector $M=(m^1,\dots,m^q)$
of $q$ smooth (regression) functions, and suppose that $n > q$.  We
would like for this to be a low rank matrix.  This suggests the
penalty
\[\|\M\|_* =
\sum_{s=1}^q \sigma_s(\M) = \sum_{s=1}^q \sqrt{\lambda_s(\M^\top
  \M)},
\]
 where $\{\lambda_s(A)\}$ denotes the eigenvalues of a
symmetric matrix $A$ and $\{\sigma_s(B)\}$ denotes the singular values
of a matrix $B$. Now, assuming the columns of $\M$ are centered, and
$\E[m^k(X)] = 0$ for each $k$, we recognize $\frac{1}{n} \M^\top \M$
as the sample covariance $\hat\Sigma(M)$ of the population covariance
\[
\Sigma(M) := \Cov(M(X)) = [\E(m^k(X) m^l(X))].
\]
This motivates the
following sample and population penalties, where $A^\half$ denotes the
matrix square root:
\begin{align}
\text{population penalty:} & \quad \|\Sigma(M)^\half\|_* = \|\Cov(M(X))^\half\|_* \\
\text{sample penalty:} &\quad \|\hat \Sigma(M)^\half\|_* = \frac{1}{\sqrt{n}} \|\M\|_*. 
\end{align}
We will also use the notation $\trnnorm{M} = \|\Cov(M(X))^\half\|_*$.

This leads to the following population
and empirical regularized risk functionals for low rank nonparametric regression:
\begin{align}
\text{population penalized risk:} & \quad \frac{1}{2} \E \|Y - M(X)\|_2^2 + \lambda \|\Sigma(M)^\half\|_* \\
\text{empirical penalized risk:} &\quad \frac{1}{2n} \| Y - \M \|_F^2 + \frac{\lambda}{\sqrt{n}} \|\M\|_*,
\end{align}
where, in the empirical case, $Y$ denotes the $n\times q$ matrix of
response values for the sample $\{(X^{(i)}, Y^{(i)})\}$.
We recall that if $A\succeq 0$ has spectral decomposition $A=UDU^{\top}$
then $A^\half = UD^\half  U^{\top}$.

\section{Constrained Rank Additive Models (\cram)}
\label{sec:additive}

We now consider the case where $X$ is $p$-dimensional.  Throughout the
paper we use superscripts to denote indices of the $q$-dimensional
response, and subscripts to denote indices of the $p$-dimensional
covariate.  We consider the family of additive models, with regression
functions of the form 
\[
m(X) = (m^1(X),\ldots, m^q(X))^\top = \sum_{j=1}^p
M_j(X_j),
\]
each term $M_j(X_j) = (m^{1}_j(X_j), \ldots,
m^q_j(X_j))^\top$ being a $q$-vector of functions evaluated at $X_j$.

In this setting we propose two different penalties.  The first
penalty, intuitively, encourages the vector 
%%family of $q$ functions
$(m^{1}_j(X_j), \ldots, m^q_j(X_j))$ to be low rank, for each
$j$.  Assume that the functions $m^k_j(X_j)$ all have mean zero; this
is required for identifiability in the additive model.  As a
shorthand, let $\Sigma_j = \Sigma(M_j) = \Cov(M_j(X_j))$ denote the
covariance matrix of the $j$-th component functions, with sample
version $\hat\Sigma_j$.  The population and sample versions of the
first penalty are then given by
\begin{align}
\label{eq:penalty1}
& \bigl\|\Sigma_1^\half\bigr\|_* + \bigl\|\Sigma_2^\half\bigr\|_* + \cdots + \bigl\|\Sigma_p^\half\bigr\|_*\\
&\bigl\|\hat\Sigma_1^\half\bigr\|_* + \bigl\|\hat\Sigma_2^\half\bigr\|_* + \cdots + \bigl\|\hat\Sigma_p^\half\bigr\|_*
   = \frac{1}{\sqrt{n}} \sum_{j=1}^p  \|\M_j\|_* .
\end{align}
The second penalty encourages the set of $q$ vector-valued
functions $(m^k_1, m^k_2, \ldots, m^k_p)^\top$ to be low rank.   This
penalty is given by 
\begin{align}
\label{eq:penalty2}
& \left\| \bigl(\Sigma_1^{1/2} \cdots \Sigma_p^{1/2} \bigr) \right\|_* \\
& \left\| \bigl(
    \hat\Sigma_1^{1/2} \cdots \hat\Sigma_p^{1/2} \bigr) \right\|_* =
    \frac{1}{\sqrt{n}} \|\M_{1:p}\|_*
\end{align}
where, for convenience of notation, $\M_{1:p} = \left( \M_1^\top 
\cdots  \M_p^\top\right)^\top$ is an $np \times q$ matrix.   The
corresponding population and empirical risk functionals, for the first penalty, are then 
\begin{align}
&\frac{1}{2} \E\Bigl\|Y -
\sum_{j=1}^p M_j(X)\Bigr\|_2^2 + \lambda \sum_{j=1}^p \bigl\|\Sigma_j^\half\bigr\|_* \\
&\frac{1}{2n} \Bigl\| Y -
\sum_{j=1}^p \M_j \Bigr\|_F^2 + \frac{\lambda}{\sqrt{n}} \sum_{j=1}^p \|\M_j\|_*
\end{align}
and similarly for the second penalty. 

Now suppose that each $X_j$ is normalized so that $\E(X_j^2)=1$.   In the
linear case we have $M_j(X_j) = X_j B_j$ where $B_j\in\reals^q$.  Let
$B = (B_1 \cdots B_p) \in \reals^{q\times p}$.
Some straightforward calculation shows that the penalties reduce to
\begin{align}
\|\Sigma_j^{1/2}\|_* &= \|B_j\|_2 \\
\|\Sigma_1^{1/2} \cdots \Sigma_p^{1/2}\|_* &= \|B\|_*.
\end{align}
Thus, in the linear case the first penalty is encouraging $B$ to be
column-wise sparse, so that many of the $B_j$s are zero, meaning that
$X_j$ doesn't appear in the fit.  This is a version of the group lasso \citep{grouplasso}.  The second penalty
reduces to the nuclear norm regularization
$\|B\|_*$ used for high-dimensional reduced-rank regression.

\section{Subdifferentials for Functional Matrix Norms}
\label{sec:subdiff}

A key to deriving algorithms for functional low-rank regression is
computation of the subdifferentials of the penalties.  We are
interested in $(q\times p)$-dimensional matrices of functions
$F=[f^{k}_j]$.  For each column index $j$ and row index $k$, $f^{k}_j$
is a function of a random variable $X_j$, and we will take
expectations with respect to $X_j$ implicitly. We write $F_j$ to mean
the $j$th column of $F$, which is a $q$-vector of functions of $X_j$.
We define the inner product between two matrices of functions as
\begin{equation}
\myllangle F,G\myrrangle :=\sum_{j=1}^p\sum_{k=1}^q \E(f^{k}_jg^{k}_j) = \sum_{j=1}^p \E(F_j^{\top}G_j)=\tr\left(\E(FG^{\top})\right)\;,
\end{equation}
and write $\norm{F}_2=\sqrt{\myllangle F,F\myrrangle}$.
Note that $\norm{F}_2$ equals the Frobenius norm of ${\sqrt{\E(FF^{\top})}}$ where $\E(FF^{\top}) =\sum_j \E(F_j F_j^{\top}) \succeq 0$
is a positive semidefinite $q\times q$ matrix. 

We define two further norms on a matrix of functions $F$, namely,
\begin{align*}
  \spnnorm{F}:=\sqrt{\Spnorm{\EE{FF^{\top}}}}=\Spnorm{\sqrt{\EE{FF^{\top}}}} \quad\text{and}\quad
  \trnnorm{F}:=
\|{\sqrt{\EE{FF^{\top}}}}\|_*,
\end{align*}
where $\spnorm{A}$ is the spectral norm (operator norm), the largest
singular value of $A$, and it is convenient to write the matrix square
root as $\sqrt{A}=A^\half$.  Each of the norms depends on $F$ only
through $\E(FF^{\top})$.  In fact, these two norms are dual---for any
$F$,
\begin{equation}
\trnnorm{F}=\sup_{\spnnormsub{G}\leq 1}\myllangle G,F\myrrangle\;,
\end{equation}
where the supremum is attained by setting
$G=\left(\sqrt{\EE{FF^{\top}}}\right)^\pinv F$, with $A^\pinv$
denoting the matrix pseudo-inverse.

%% We will define two further norms on a matrix of functions $F$, each of
%% which depends on $F$ only through $\E(FF^{\top})$.  We first define
%% the norm $\spnnorm{F}:=\sqrt{\Spnorm{\EE{FF^{\top}}}}$, where
%% $\spnorm{A}$ is the spectral norm (operator norm), the largest
%% singular value of $A$.  The second norm is then $\trnnorm{F}:=
%% \|{\sqrt{\EE{FF^{\top}}}}\|_*$, where $\sqrt{A}=A^\half$ denotes the
%% matrix square root of a matrix $A\succeq 0$, that is, if $A$ has
%% spectral decomposition $A=UDU^{\top}$ then $\sqrt{A} = U \sqrt{D}
%% U^{\top}$.  The second norm is then $\trnnorm{F}:=
%% \|{\sqrt{\EE{FF^{\top}}}}\|_*$.  These two norms are dual---for any
%% $F$,
%% \begin{equation}
%% \trnnorm{F}=\sup_{\spnnormsub{G}\leq 1}\myllangle G,F\myrrangle\;,
%% \end{equation}
%% where the supremum is attained by setting
%% $G=\left(\sqrt{\EE{FF^{\top}}}\right)^\pinv F$, with $A^\pinv$
%% denoting the matrix pseudo-inverse.

\begin{proposition}
The subdifferential of $\trnnorm{F}$ is the set
\begin{equation}
\label{Subdiff}
\mathcal{S}(F) := \left\{  \Bigl(\sqrt{\EE{FF^{\top}}}\Bigr)^\pinv F + H \
  :\  \spnnorm{H} \leq 1,\ \EE{FH^{\top}}=\mathbf{0}_{q\times q}, \ \EE{FF^{\top}} H = \mathbf{0}_{q\times p}
\; a.e.\right\}\;.
\end{equation}
\end{proposition}

\begin{proof}
The fact that $\mathcal{S}(F)$ contains the subdifferential $\partial
\trnnorm{F}$ can be proved by comparing our setting (matrices of
functions) to the ordinary matrix case; see \cite{watson:92,recht:2010}.  Here, we show
the reverse inclusion, $\mathcal{S}(F) \subseteq
\partial \trnnorm{F}$.
Let $D\in\mathcal{S}(F)$ and let $G$ be any element of the function space. We need to show
\begin{equation}
\label{Claim1}
\trnnorm{F+G} \geq \trnnorm{F} +\myllangle G,  D\myrrangle\;,
\end{equation}
where $
D = \Bigl(\sqrt{\EE{FF^{\top}}}\Bigr)^\pinv F + H =: \tilde{F}+H$
for some $H$ satisfying the conditions in \eqref{Subdiff} above.
Expanding the right-hand side of \eqref{Claim1}, we have
\begin{align}
\trnnorm{F}+\myllangle G, D\myrrangle &= \trnnorm{F} + \myllangle G,
\tilde{F}
+ H \myrrangle \\
&= \myllangle F+G,  \tilde{F}+ H \myrrangle \\
&\leq  \trnnorm{F+G} \spnnorm{D}\;,
\end{align}
where the second equality follows from $\trnnorm{F}= \myllangle  F,
\tilde{F}\myrrangle$, 
and the fact that  $\myllangle F, H\myrrangle =\tr\bigl(\EE{FH^{\top}}\bigr)=0$.
The inequality follows from the duality of the norms. 

Finally, we show that $\spnnorm{D}\leq 1$. We have
\begin{align}
\EE{DD^{\top}} &= \EE{\tilde{F}\tilde{F}^{\top}} + \EE{\tilde{F}H^{\top}}+\EE{H\tilde{F}^{\top}}+\EE{HH^{\top}}
\\
& =  \EE{\tilde{F}\tilde{F}^{\top}} +\EE{HH^{\top}} \;,
\end{align}
 where we use the fact that $\EE{FH^{\top}}=\mathbf{0}_{q\times q}$, implying $\EE{\tilde{F}H^{\top}}=\mathbf{0}_{q\times q}$.
 Next, let $\EE{FF^{\top}}=VDV^{\top}$ be a reduced singular value decomposition, where $D$ is a positive diagonal matrix of size $q'\leq q$. Then $\EE{\tilde{F}\tilde{F}^{\top}}=VV^{\top}$, and we have
 \[\EE{FF^{\top}}\cdot H=\mathbf{0}_{q\times p}\text{ a.e.} \ \Leftrightarrow \ V^{\top}H=\mathbf{0}_{q'\times p}\text{ a.e.} \ \Leftrightarrow \ \EE{\tilde{F}\tilde{F}^{\top}}H=\mathbf{0}_{q\times p}\text{ a.e.}\;.\]
This implies that $\EE{\tilde{F}\tilde{F}^{\top}}\cdot\EE{HH^{\top}}=\mathbf{0}_{q\times q}$ and so these two symmetric matrices have orthogonal row spans and orthogonal column spans. Therefore,
\begin{align}
\Spnorm{\EE{DD^{\top}}}
&=\Spnorm{\EE{\tilde{F}\tilde{F}^{\top}}+\EE{HH^{\top}}} \\
&=
\max\left\{\Spnorm{\EE{\tilde{F}\tilde{F}^{\top}}},\Spnorm{\EE{HH^{\top}}}\right\}\\
&\leq 1\;,
\end{align}
where the last bound comes from the fact that $\spnnorm{\tilde{F}},\spnnorm{H}\leq 1$. Therefore $\spnnorm{D}\leq 1$.
\end{proof}

This gives the subdifferential of penalty 2, defined in \eqref{eq:penalty2}.  We can view the first penalty 
update as just a special case of the second penalty update. For
penalty 1 in \eqref{eq:penalty1}, if we are updating $F_j$ and fix all the other functions, we are now penalizing the norm
\begin{equation}
\trnnorm{F_j}=\Trnorm{\sqrt{\EE{F_jF_j^{\top}}}}\;,
\end{equation}
which is clearly just a special case of penalty 2 with
a single $q$-vector of functions instead of $p$ different
$q$-vectors of functions. So, we have
\begin{equation*}
\label{Subdiff_penalty1}\partial\trnnorm{F_j}
= \left\{  \left(\sqrt{\EE{F_j F_j^{\top}}}\right)^\pinv \! F_j + H_j \ :\
  \spnnorm{H_j} \leq 1,\ \EE{F_j H_j^{\top}}=\mathbf{0}, \ \EE{F_jF_j^{\top}} H_j = \mathbf{0}\; a.e.\right\}.
\end{equation*}

\section{Stationary Conditions and Backfitting Algorithms}
\label{sec:stationary}

Returning to the base case of $p=1$ covariate, consider the population
regularized risk optimization
\begin{equation}
\label{eq:risk}
\min_M \Bigl\{ \frac{1}{2} \E\|Y - M(X)\|_2^2 + \lambda \trnnorm{M}\Bigr\},
\end{equation}
where $M$ is a vector of $q$ univariate functions.
%%  as $m(X) = (m^1(X),\ldots, m^q(X))^\top$ is a $q$-vector of regression
%% functions.  
The stationary condition for this optimization is
\begin{equation}
\label{eq:stationary}
\E(Y \given X) = M(X) + \lambda V(X) \quad \text{a.e.} \;\;
\text{for some $V\in \partial \trnnorm{M}$}.
\end{equation}
Define $P(X) := \E(Y\given X)$.

\begin{proposition}
Let $\E(PP^\top) = U \diag(\tau) U^\top$ be the singular value decomposition and define
\begin{equation}
M = U \diag([1-\lambda/\sqrt{\tau}]_+) U^\top P
\end{equation}
where $[x]_+ = \max(x,0)$.
Then $M$ satisfies stationary condition \eqref{eq:stationary},
and is a minimizer of \eqref{eq:risk}.
\end{proposition}

\begin{proof}
Assume the singular values are sorted as $\tau_1 \geq \tau_2\geq
\cdots \geq \tau_q$, and let $r$ be the largest index such that
$\sqrt{\tau_r} > \lambda$. Thus, $M$ has rank $r$. Note that
$\sqrt{\E(MM^\top)} = U \diag([\sqrt{\tau}-\lambda]_+) U^\top$, and
therefore 
\begin{equation}
\lambda \bigl(\sqrt{\E(MM^\top)}\bigr)^\pinv M = U 
\diag(\lambda/\sqrt{\tau_{1:r}},\mathbf{0}_{q-r}) U^\top P
\end{equation}
where $x_{1:k} = (x_1,\ldots, x_k)$ and $c_k = (c,\ldots, c)$.
It follows that
\begin{equation}
M + \lambda \bigl(\sqrt{\E(MM^\top)}\bigr)^\pinv M = U \diag(\mathbf{1}_r, \mathbf{0}_{q-r}) U^\top P.
\end{equation}
Now define
\begin{equation}
H = \frac{1}{\lambda} U \diag(\mathbf{0}_r, \mathbf{1}_{q-r}) U^\top P
\end{equation} 
and take $V = \bigl(\sqrt{\E(MM^\top)}\bigr)^\pinv M + H$.  Then we have
$M + \lambda V = P$.  

It remains to show that $H$ satisfies
the conditions of the subdifferential in \eqref{Subdiff}.
Since 
\begin{equation}
\sqrt{\E(HH^\top)} = U\diag(\mathbf{0}_r,
\sqrt{\tau_{r+1}}/\lambda,\ldots, \sqrt{\tau_{q}}/\lambda) U^\top
\end{equation}
we have
$\spnnorm{H}\leq 1$.  Also, $\E(MH^\top) = \mathbf{0}_{q\times q}$ since
\begin{equation}
\diag(1-\lambda/\sqrt{\tau_{1:r}}, \mathbf{0}_{q-r}) \diag(\mathbf{0}_r,
\mathbf{1}_{q-r}/\lambda) = \mathbf{0}_{q\times q}.
\end{equation}
Similarly, $\E(MM^\top) H = \mathbf{0}_{q\times q}$ since
\begin{equation}
\diag((\sqrt{\tau_{1:r}}-\lambda)^2, \mathbf{0}_{q-r}) \diag(\mathbf{0}_r,
\mathbf{1}_{q-r}/\lambda) = \mathbf{0}_{q\times q}.
\end{equation}
It follows that $V\in \partial\spnnorm{M}$.
\end{proof}

\begin{figure}[t]
{\sc \cram{} Backfitting Algorithm --- First Penalty\hfill}
\vskip5pt
\begin{center}
\hrule
\vskip7pt
\normalsize
\begin{enumerate}
\item[] \textit{Input}:  Data matrices
${X}\in\mathbb{R}^{n\times p}$ and
${Y}\in\mathbb{R}^{n\times q}$, regularization parameter $\lambda\ge 0$.
\item[] \textit{Initialize} $\widehat{\mathbb{M}}_j=\big(\widehat{m}^k_j(X_{ij})\big)=0\in\mathbb{R}^{n\times q}$, for $j=1,\ldots, p$.
\item[] \textit{Iterate} until convergence:
\begin{enumerate}
   \item[] \textit{For each $j=1,\ldots, p$}:
   \begin{enumerate}
        \item[(1)] Compute the residual $Z_j \leftarrow  Y -
          \sum_{j'\neq j} \widehat{\mathbb{M}}_{j'}$.
        \item[(2)] Estimate $\Proj_j = \E[\Res_j\given
          X_j]$ by smoothing: $\ \hat \Proj_j = \S_j \Res_j$. \\[5pt]
        \item[(3)] Compute SVD: $\ \frac{1}{n}\hat P_j^\top \hat P_j = U \diag(\tau) U^\top$.
        \item[(4)] Soft threshold: $\ \widehat{\mathbb{M}}_j \leftarrow  \widehat{P}_j\, U \diag(\left[1 - {\lambda}/{\sqrt{\tau}}\right]_+) U^{\top}$.
        \item[(5)] Center: $\ \widehat{\mathbb{M}}_j \leftarrow \widehat{\mathbb{M}}_j - \mathrm{mean}(\widehat{\mathbb{M}}_j)$.
   \end{enumerate}
\end{enumerate}
\item[] \textit{Output}: Component functions $\widehat{\mathbb{M}}_j$
  and estimates of the conditional mean vector $\sum_j \widehat{\mathbb{M}}_{ij}$.
\end{enumerate}
\vskip3pt
\hrule
\end{center}
\vskip0pt
\caption{The \cram{} backfitting algorithm, using the first penalty,
  which penalizes each component.\label{fig:algo}
}
\end{figure}

The analysis above justifies a backfitting algorithm for estimating
a constrained rank additive model with the first penalty, where the
objective is
\begin{equation}
\min_{M_j} \Bigl\{ \frac{1}{2} \E\Bigl\|Y - \sum_{j=1}^p
M_j(X_j)\Bigr\|_2^2 + \lambda \sum_{j=1}^p \trnnorm{M_j}\Bigr\}.
\end{equation}
For a given coordinate $j$, we form the residual $Z_j = Y- \sum_{k\neq j} M_k$, 
and then compute the projection $P_j = \E(Z_j\given X_j)$, with 
singular value decomposition $\E(P_j P_j^\top) = U\diag(\tau)
U^\top$.  We then update
\begin{equation}
M_j = U \diag([1-\lambda/\sqrt{\tau}]_+) U^\top P_j
\end{equation}
and proceed to the next variable.  This is a Gauss-Seidel procedure
that parallels the population backfitting algorithm for \spam{}
\citep{Ravikumar:09}.  

In the sample version we replace the
conditional expectation $\Proj_j = \E(Z_j\given X_j)$ by
a nonparametric linear smoother, $\hat\Proj_j = \S_j Z_j$.  The
algorithm is given in Figure~\ref{fig:algo}.  The algorithm for
penalty 2 is similar and given in Figure~\ref{fig:algo2}.  Both
algorithms can be viewed as functional projected gradient descent
procedures.  Note that to predict
at a point $x$ not included in the training set, the
smoother matrices are constructed using that point; that is,
$\hat P_j(x_j) = S_j(x_j)^\top Z_j$.

\begin{figure}[t]
{\sc \cram{} Backfitting Algorithm --- Second Penalty\hfill}
\vskip5pt
\begin{center}
\hrule
\vskip7pt
\normalsize
\begin{enumerate}
\item[] \textit{Input}:  Data matrices
${X}\in\mathbb{R}^{n\times p}$ and
${Y}\in\mathbb{R}^{n\times q}$, regularization parameter $\lambda\ge 0$.
\item[] \textit{Initialize} $\widehat{\mathbb{M}}_j=\big(\widehat{m}^k_j(X_{ij})\big)=0\in\mathbb{R}^{n\times q}$, for $j=1,\ldots, p$.
\item[] \textit{Iterate} until convergence:
\begin{enumerate}
   \item[] \textit{For each $j=1,\ldots, p$}:
   \begin{enumerate}
        \item[(1)] Compute the residual $Z_j \leftarrow  Y -
          \sum_{j'\neq j} \widehat{\mathbb{M}}_{j'}$.
        \item[(2)] Estimate $\Proj_j = \E[\Res_j\given
          X_j]$ by smoothing: $\ \hat \Proj_j = \S_j \Res_j$. \\[5pt]
        \item[(3)] Compute SVD: $\ \frac{1}{n}\hat P_{1:p}^\top \hat P_{1:p} = U \diag(\tau) U^\top$.
        \item[(4)] Soft threshold:  $\ \widehat{\mathbb{M}}_{1:p}\leftarrow  \hat
        \Proj_{1:p}\, U \diag(\left[1 - {\lambda}/{\sqrt{\tau}}\right]_+)U^\top $.
        \item[(5)] Center: $\ \widehat{\mathbb{M}}_j \leftarrow \widehat{\mathbb{M}}_j - \mathrm{mean}(\widehat{\mathbb{M}}_j)$.
   \end{enumerate}
\end{enumerate}
\item[] \textit{Output}: Component functions $\widehat{\mathbb{M}}_j$ and estimates of the conditional mean vector $\sum_j \widehat{\mathbb{M}}_{ij}$.
\end{enumerate}
\vskip3pt
\hrule
\end{center}
\vskip0pt
\caption{The \cram{} backfitting algorithm, using the second penalty,
  which penalizes the components together.\label{fig:algo2}}
\end{figure}

\section{Working over an RKHS}

Suppose that the functions $m_{j}^k$ are required to lie in a Hilbert
space $\H_j$.  A modified empirical optimization is (for the first penalty)
\begin{align}
\Bigl\|Y - \sum_{j=1}^p \mathbb{M}_j \Bigr\|_F^2 + \lambda_n \sum_{j=1}^p \|\mathbb{M}_j\|_* +
\rho_n \sum_{j=1}^p \sum_{k=1}^q \|m^k_{j}\|_{\H_j}
\label{opt}
\end{align}
where $Y$ is an $n\times q$ data matrix and $\mathbb{M}_j$ is an $n
\times q$ matrix of function values associated with the $j$th columns
of a $n\times p$ data matrix $X$.  The first penalty is then a
nuclear-norm constraint on these observed function values. The second
penalty is a smoothness penalty on each of the coordinate functions in
the appropriate Hilbert space, and is not empirical.

If $\H_j$ is an RKHS with kernel $\K_j$, then the representer theorem 
implies that we can restrict to functions $m_{j}^k$ of the form
\begin{align}
m^k_{j}(\,\cdot\,) = \sum_{i=1}^n \alpha^k_{ij} K_j(x_{ij}, \,\cdot\,),
\end{align}
where the $\alpha^k_{ij}$ are real weights.
In this case the optimization becomes a finite dimensional semidefinite
program over $\alpha$.
This parallels the approach of \cite{Raskutti:10} for sparse additive
models; see also \cite{Dinuzzo:2011}.

If $K_j = \left[K_{h_j}(x_{ij}, x_{i'j})\right] \in\reals^{n\times n}$ denotes
the Gram matrix for the $j$th variable, then $F_j = K_j\alpha_j$
where $\alpha_j  = \left[\alpha_{ij}^k\right] \in\reals^{n\times q}$.
Using the first penalty, the convex optimization is then
\begin{equation}
\min_\alpha   \frac{1}{2n} \Bigl\|Y - \sum_j K_j\alpha_j\Bigr\|_F^2 + \lambda_n \sum_{j=1}^p
\|K_j \alpha_j\|_* + \rho_n \sum_{k=1}^q \sum_{j=1}^p
\sqrt{\alpha_j^{kT} K_j \alpha_j^k}
\end{equation}
where the third term is a smoothness penalty for the RKHS.  This
is a cone program with constraints involving both the second-order
cone and the semidefinite cone.   
%(Not sure if there is a standard name for
%this...)

\section{Excess Risk Bounds}
\label{sec:excess}

The population risk of a $q\times p$ regression matrix $M(X) =
[M_1(X_1) \cdots M_p(X_p)]$ is
\begin{equation*}
R(M) = \E \| Y - M(X) \ones_p \|_2^2,
\end{equation*}
with sample version denoted $\hat R(M)$.
Consider all models that can be written as
\begin{equation*}
M(X) = U \cdot D \cdot V(X)^\top
\end{equation*}
where $U$ is an orthogonal $q\times r$ matrix, $D$
is a positive diagonal matrix, and $V(X) = [v_{js}(X_j)]$ satisfies
$\E(V^\top V) = I_r$.
The population risk can be reexpressed as
\begin{align*}
R(M) 
%%& = \E \|Y - M(X) \ones_p\|_2^2 \\
     & = \tr \left\{\begin{pmatrix}-I_q\\DU^{\top}\end{pmatrix}^\top \E\left[ \left(\begin{array}{c}Y\\ V(X)^\top\end{array}\right)
     \left(\begin{array}{c}Y\\ V(X)^\top\end{array}\right)^\top
     \right] \begin{pmatrix}-I_q\\ DU^{\top}\end{pmatrix} \right\} \\
     & = \tr \left\{\begin{pmatrix}-I_q\\DU^{\top}\end{pmatrix}^\top 
     \begin{pmatrix} \Sigma_{YY} & \Sigma_{YV} \\ \Sigma_{YV}^\top &
       \Sigma_{VV} \end{pmatrix}
     \begin{pmatrix}-I_q\\DU^{\top}\end{pmatrix} \right\}
\end{align*}
and similarly for the sample risk, with $\hat\Sigma_n(V)$ replacing
$\Sigma(V) := \Cov((Y,V(X)^{\top}))$ above.  The ``uncontrollable''
contribution to the risk, which does not depend on $M$, is
$R_u = \tr \{\Sigma_{YY}\}$.
We can express the remaining ``controllable'' risk as
\begin{align}\notag
R_c(M) =R(M) - R_u
     & = \tr \left\{\begin{pmatrix}-2I_q\\DU^{\top}\end{pmatrix}^\top 
     \Sigma(V)
     \begin{pmatrix}\mathbf{0}_q\\DU^{\top}\end{pmatrix} \right\}.
\end{align}
Using the von Neumann trace inequality, $\tr(AB) \leq \|A\|_p \|B\|_{p'}$
where $1/p + 1/p' = 1$,
\begin{align}
R_c(M) - \hat R_c(M) 
\notag&\leq \spnorm{\begin{pmatrix}-2I_q\\DU^{\top}\end{pmatrix}^\top (\Sigma(V) - \hat\Sigma_n(V))}
     \trnorm{ \begin{pmatrix}\mathbf{0}_q\\DU^{\top}\end{pmatrix}} \\
\notag&\leq \spnorm{\begin{pmatrix}-2I_q\\DU^{\top}\end{pmatrix}^\top}\spnorm{\Sigma(V) - \hat\Sigma_n(V)}
     \trnorm{D} \\
\notag
& \leq C \max(2, \spnorm{D}) \, \spnorm{\Sigma(V) - \hat \Sigma_n(V)} \, \trnorm{D} \\
\label{eq:b}
& \leq C \max\{2, \trnorm{D}^2\} \, \spnorm{\Sigma(V) - \hat \Sigma_n(V)} 
\end{align}
where here and in the following $C$ is a generic constant.  For the
last factor in (\ref{eq:b}), it holds that
\begin{align}
\notag\sup_V \spnorm{\Sigma(V)-\hat\Sigma_n(V)} 
\leq C \sup_V\sup_{w\in\cN} w^\top \left(\Sigma(V)  - \hat\Sigma_n(V)\right) w 
\end{align}
where $\cN$ is a $1/2$-covering of the unit $(q+r)$-sphere, which has
size $|\cN| \leq 6^{q+r} \leq 36^q$; compare \citet[p.~665]{Vershynin:cov}.  We now
assume that the functions $v_{sj}(x_j)$ are uniformly bounded from a
Sobolev space of order two.  Specifically, let $\{\psi_{jk}:
k=0,1,\ldots \}$ denote a uniformly bounded, orthonormal basis with
respect to $L^2[0,1]$, and assume that $v_{sj}\in \cT_j$ where
\begin{equation*}
\cT_j = \Bigl\{ f_j:\ 
f_j(x_j) =\sum_{k=0}^\infty a_{jk}\psi_{jk}(x_j),\ \ \ 
\sum_{k=0}^\infty a_{jk}^2 k^{4}\leq K^2 \Bigr\}
\end{equation*}
for some $0 < K < \infty$.  
The $L_\infty$-covering number of $\cT_j$ satisfies $\log
\cN(\cT_j, \epsilon) \leq K/\sqrt{\epsilon}$.

Suppose that $Y - \E(Y\given X) = W$ is Gaussian and the true
regression function $\E(Y\given X)$ is bounded.
Then the family of random variables $Z_{(V,w)} := \sqrt{n}\cdot  w^\top(\Sigma(V) -
\hat\Sigma_n(V)) w$ is sub-Gaussian and sample continuous.  It
follows from a result of \cite{cesa:99} that 
\begin{align*}
\E\left(\sup_V \sup_{w\in \cN} w^\top (\Sigma(V) - \hat\Sigma_n(V)) w\right) &\leq
\frac{C}{\sqrt{n}} \int_0^B \sqrt{q\log(36) + \log(pq) +
  \frac{K}{\sqrt{\epsilon}}} \ d\epsilon
\end{align*}
for some constant $B$.  Thus, by Markov's inequality we
conclude that 
\begin{equation}\label{eq:OP}
\sup_V \spnorm{\Sigma(V) - \hat\Sigma_n(V)} = O_P\left(\sqrt{\frac{q + \log(pq)}{n}}\right),
\end{equation}
when $n$ tends to infinity and $q$ and $p$ possibly change with $n$.  If
\[
\trnnorm{M} = \trnorm{D} = o\left(\frac{n}{(q +\log(pq))}\right)^{{1/4}},
\]
then returning to \eqref{eq:b}, this gives us a bound on $R_c(M) -
\hat R_c(M) $ that is $o_P(1)$. More precisely, for $\beta_n>0$, define
the class of matrices of functions
\begin{equation}
{\mathcal M}(\beta_n) = \left\{M \,:\,  M(X) = UDV(X)^\top,\; \text{with}\ \E(V^\top V) = I,\; v_{sj}\in \cT_j,\; \trnorm{D} \le\beta_n \right\}.
\end{equation}
%% \begin{equation*}
%% {\mathcal M}_n = \left\{M \,:\,  M(X) = UDV(X)^\top,\; \text{with}\ \E(V^\top V) = I,\; v_{sj}\in \cT_j,\; \trnorm{D} = o\left(\frac{n}{q +\log(pq)}\right)^{{1/4}}\right\}.
%% \end{equation*}
Then, for a fitted matrix $\hat M$ chosen from ${\mathcal M}(\beta_n)$, writing $M_*=\arg\min_{M\in\mathcal{M}(\beta_n)}R(M)$, we have
\begin{align*}
R(\hat M) -  \inf_{M\in\mathcal{M}(\beta_n)} R(M) &= R(\hat M) - \hat R(\hat M) - (R(M_*) - \hat R(M_*)) + (\hat R(\hat M) -
\hat R(M_*)) \\
&\leq \big[ R(\hat M) - \hat R(\hat M)\big] - \big[R(M_*) - \hat R(M_*)\big].
\intertext{Subtracting $R_u-\hat{R_u}$ from each of the bracketed
  differences, we obtain that}
R(\hat M) -  \inf_{M\in\mathcal{M}(\beta_n)} R(M) &\le  \big[R_c(\hat M) - \hat R_c(\hat M)\big] - \big[R_c(M_*) - \hat R_c(M_*)\big]\\
& \leq 2 \sup_{M\in\mathcal{M}(\beta_n)} \left\{R_c(M) - \hat R_c(M)\right\} \\
&\stackrel{\text{by \eqref{eq:b}}}{\leq} O_P\left(\trnorm{D}^2 \,
  \spnorm{\Sigma(V) - \hat \Sigma_n(V)}\right).
\end{align*}
Now if 
\begin{equation}
  \label{eq:beta-n}
\beta_n= o\left(\frac{n}{q +\log(pq)}\right)^{{1/4}},
\end{equation}
then we may conclude from \eqref{eq:OP} that
\begin{align*}
R(\hat M) -  \inf_{M\in\mathcal{M}(\beta_n)} R(M) &=
o_P(1). 
\end{align*}
This proves the following result.
\begin{proposition}
  Let $\hat M$ minimize the empirical risk $\frac{1}{n} \sum_i \|Y_i -
  \sum_j M_j(X_{ij})\|^2_2$ over the class $\mathcal{M}(\beta_n)$.
  Suppose that $Y - \E(Y\given X)$ is Gaussian, the true regression
  function $\E(Y\given X)$ is bounded, and $\beta_n$
  satisfies~(\ref{eq:beta-n}) as $n\to\infty$.  Then it holds that
  \[R(\hat M) -  \inf_{M\in \mathcal{M}(\beta_n)} R(M)
  \stackrel{P}{\longrightarrow} 0\;.\]
\end{proposition}

\section{Examples}
\label{sec:genes}

\subsection{Synthetic Data Example}

\begin{figure}[t]
\begin{center}
\begin{tabular}{cc}
Penalty 1 & Penalty 2 \\
\hskip-15pt
\includegraphics[width=.45\textwidth]{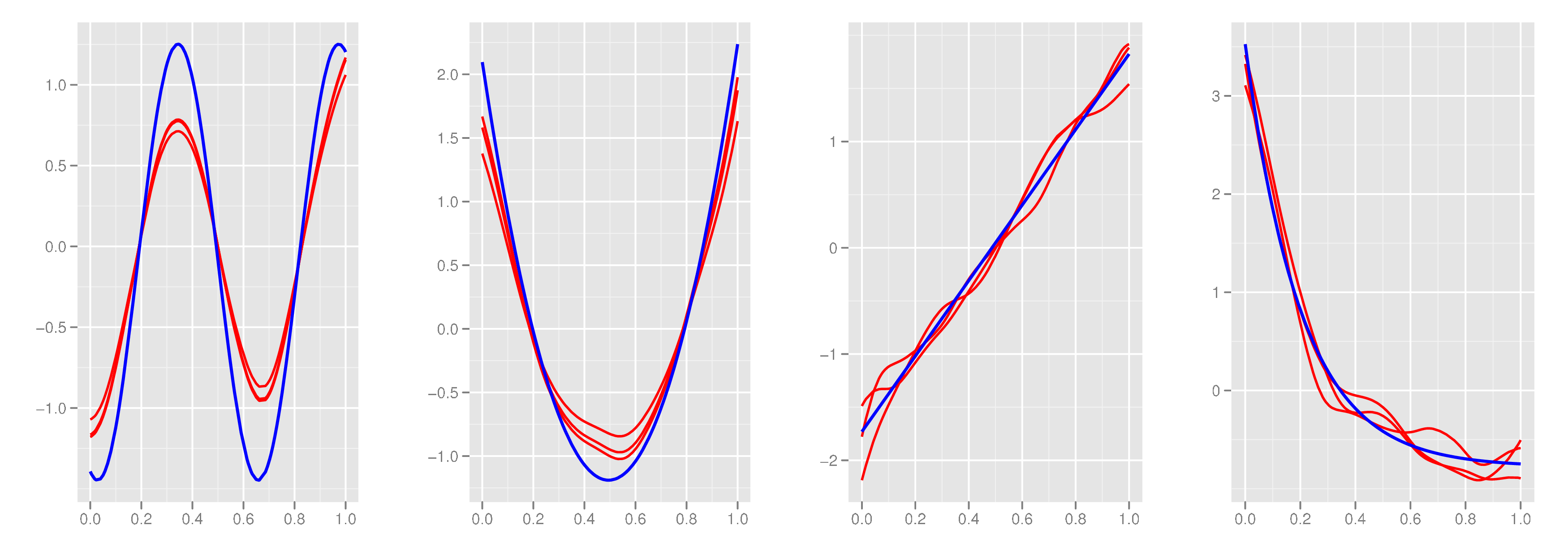} &
\qquad
\includegraphics[width=.45\textwidth]{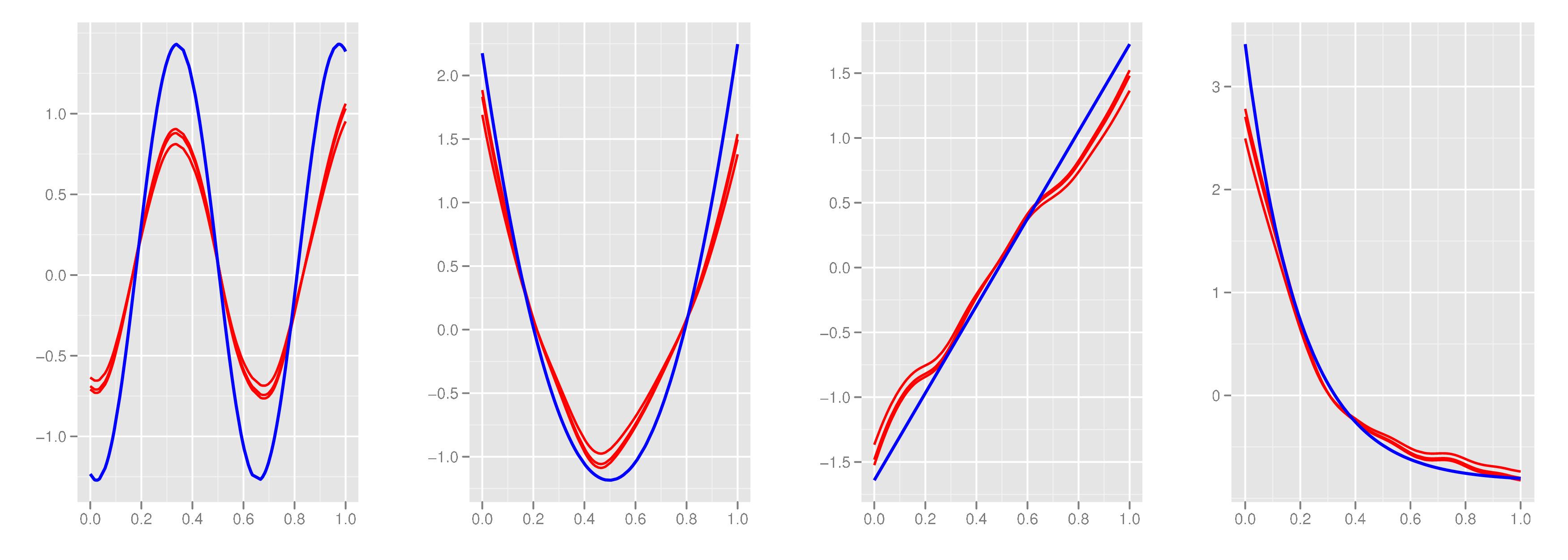}
\end{tabular}
\caption{Example of the constrained-rank backfitting algorithm, with $q=3$
  and $p=4$.  The model is the same for each $k$; that is,
  $m_{j}^{k} = m_{j}^{k'}$ for each $j, k, k'$. The left
  plots show the fits with penalties $\lambda_j \trnnorm{M_j}$, with 
  $\lambda = (3, 3, 0, 0)$. The right plots use 
  penalty 2, and the solution has rank 1. \label{synthfig}}
\end{center}
\end{figure}

Figure~\ref{synthfig} 
shows an example of the backfitting algorithms for penalties 1 and 2.
For this example $q=3$ and $p=4$.  
The model is the same for each $k$; that is,
  $m_{j}^{k} = m_{j}^{k'}$ for each $j, k, k'$.  
The true regression functions are $m_1^k(x) = \sin(2x)$, $m_2^k(x) =
x^2-c_2$, $m_3^k(x) = x$, and $m_4^k(x) = e^{-x}-c_4$, where $c_2$ and
$c_4$ are centering constants.  The left
  group of plots in Figure~\ref{synthfig} shows the fits obtained with penalties $\lambda_j \trnnorm{M_j}$ with
  regularization parameters $\lambda = (\lambda_1, \lambda_2,
  \lambda_3, \lambda_4) = (3, 3, 0, 0)$, so that the
  third and fourth functions are not regularized.   The plots show how
  the first and second function estimates are identical up to
  multiplicative scaling for each   $k=1,2,3$ ($\mathbb{M}_1$ and $\mathbb{M}_2$ have rank one), while the third and fourth function estimates
  vary ($\mathbb{M}_3$ and $\mathbb{M}_4$ have rank three).  The fits are made with
  local linear smoothing.  The right set of plots corresponds to 
  penalty 2, and the solution has rank 1.  We intentionally use a bandwidth that is too small, to better show the
  differences with and without regularization. 
The true regression 
functions $m_j^k$ are shown in blue in the plots; the fitted functions
are in red, with the fits for a given $j$ superimposed on the same
plot. The sample size is $n=150$, and the noise variance is $\sigma^2
= 1$.

\subsection{Gene Expression Data}

To further illustrate the proposed nonparametric reduced rank regression
techniques, we consider data on gene expression in \textit{E.~coli}
from the ``DREAM 5 Network Inference Challenge''\footnote{\tt
  http://wiki.c2b2.columbia.edu/dream/index.php/D5c4} \citep{wisdom}.
In this challenge genes were classified as transcription factors (TFs)
or target genes (TGs).  Transcription factors regulate the target
genes, as well as other TFs.

\begin{figure}[!t]
\begin{center}
%\vskip-1.6in
\begin{tabular}{cc}
\small Penalty 1, $\lambda=20$ & 
\small Penalty 2, $\lambda=5$ \\[-0pt]
\hskip-3pt
\includegraphics[scale=.30,angle=-90]{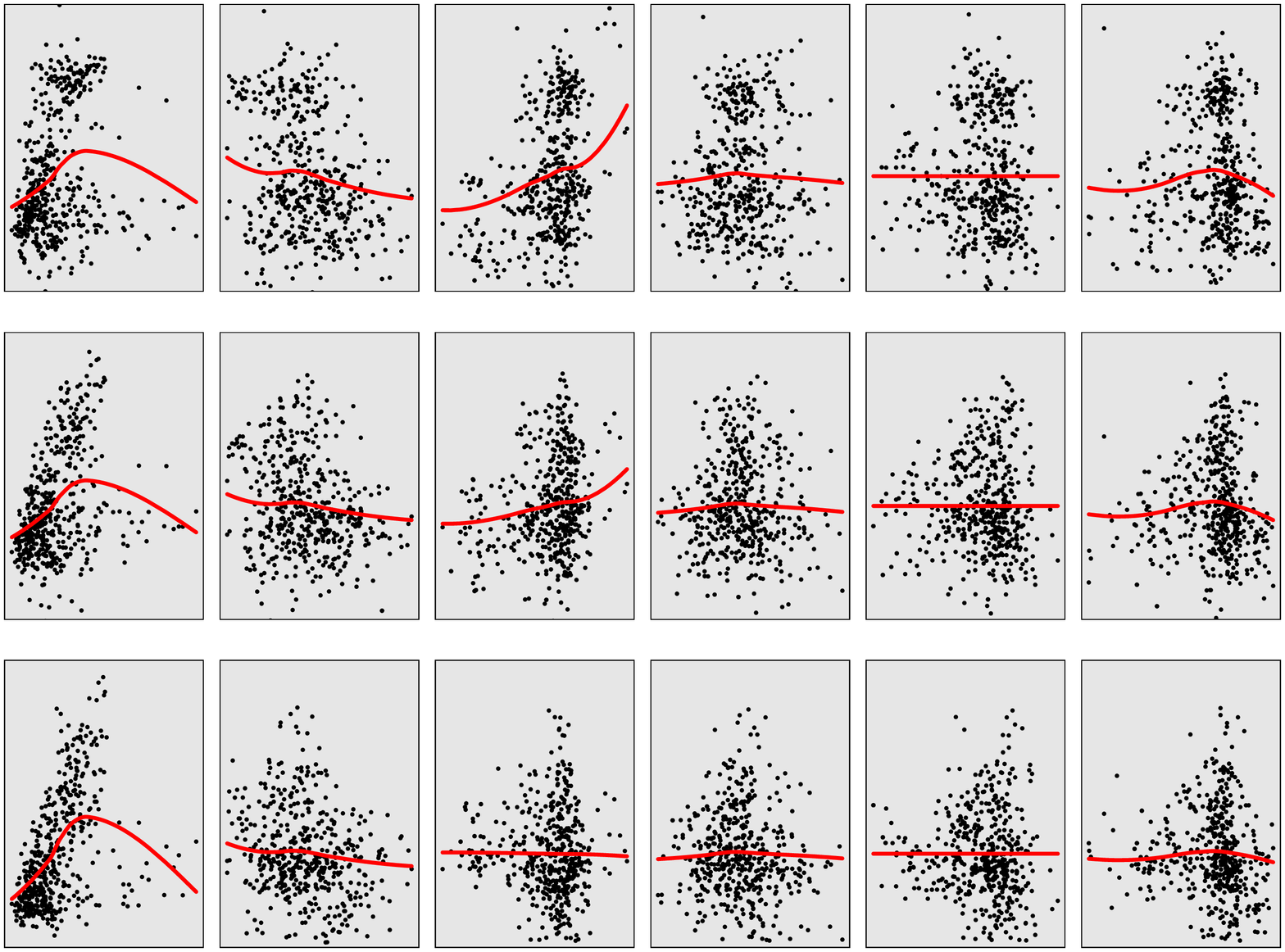}&
\includegraphics[scale=.30,angle=-90]{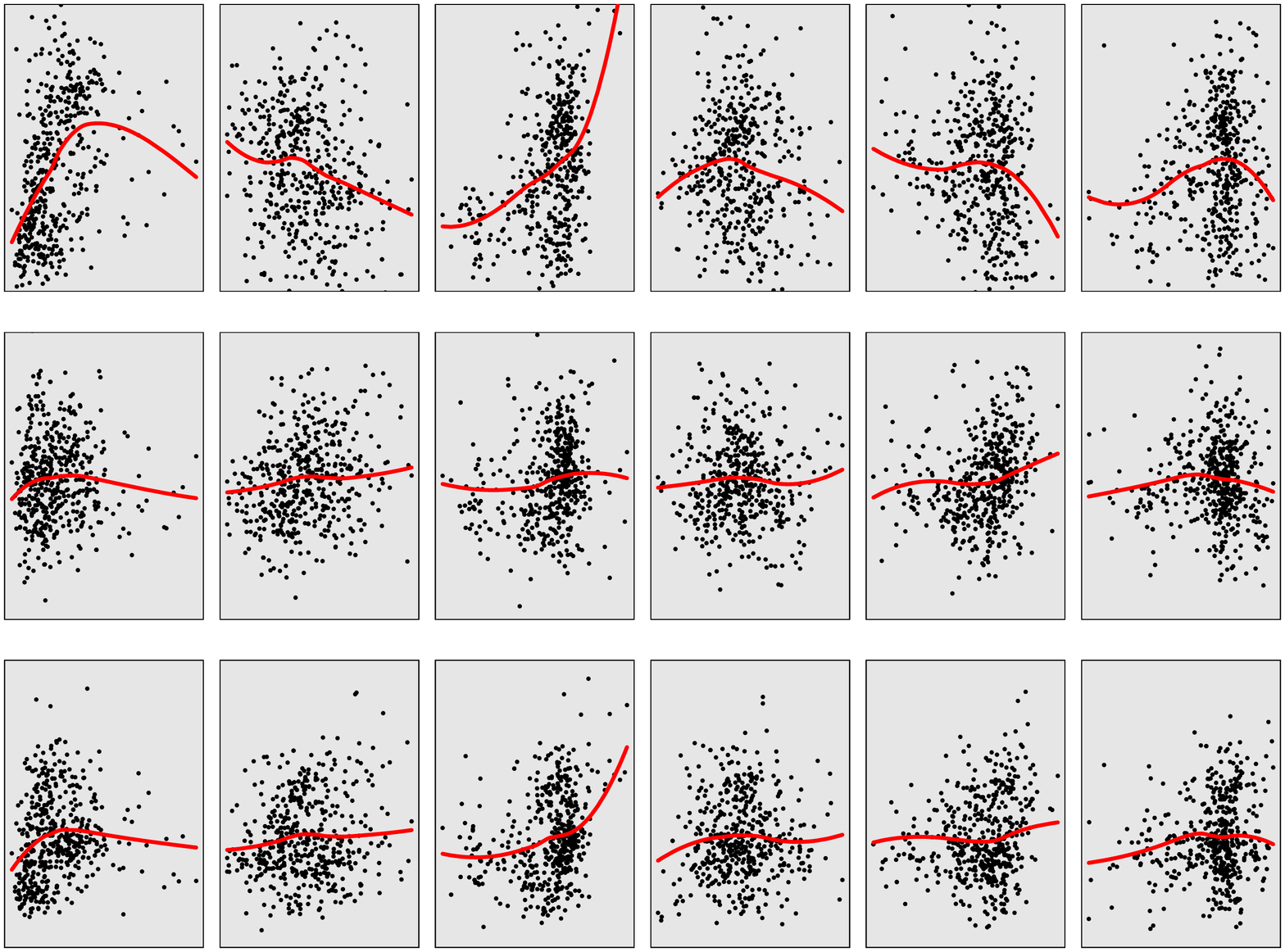}
\end{tabular}
\caption{\label{f01}
  Fits on the gene expression data.  Left: Penalty 1 with
  large tuning parameter. Right: Penalty 2 tuned via 10-fold
  cross-validation. Plotted points are residuals holding out the given
  predictor.}
\end{center}
\vskip-10pt
\end{figure}

We focus on predicting the expression levels $Y$ for a particular set
of $q=27$ TGs, using the expression levels $X$ for $p=6$ TFs.  Our
motivation for analyzing these 33 genes is that, according to the gold
standard gene regulatory network used for the DREAM 5 challenge, the 6
TFs form the parent set common to two additional TFs, which have the
27 TGs as their child nodes.  In fact, the two intermediate nodes
d-separate the 6 TFs and the 27 TGs in a Bayesian network
interpretation of this gold standard.  This means that if we treat the
gold standard as a causal network, then up to noise, the functional
relationship between $X$ and $Y$ is given by the composition of a map
$g:\mathbb{R}^6\to\mathbb{R}^2$ and a map $h:\mathbb{R}^2\to
\mathbb{R}^{27}$.  If $g$ and $h$ are both linear, their composition
$h\circ g$ is a linear map of rank at most than 2.  As observed in
Section~\ref{sec:penalties}, such a reduced rank linear model is a
special case of an additive model with reduced rank in the sense of
penalty 2.  More generally, if $g$ is an additive function and $h$ is
linear, then $h\circ g$ has rank at most 2 in the sense of penalty 2.
Higher rank can in principle occur under functional composition, since
even a univariate additive map $h:\mathbb{R}\to\mathbb{R}^q$ may have
rank up to $q$ under our penalties (penalty 1 and 2
coincide for univariate maps).
%%    Trivially, we may take $g(x)=(x_1,x_2)$ to be
%% the projection on the first two coordinates and $h$ any additive
%% function.

The backfitting algorithm of Figure~\ref{fig:algo} with penalty 1 and
a rather aggressive choice of the tuning parameter $\lambda$ produces
the estimates shown in Figure \ref{f01}, for which we have selected
three of the 27 TGs.  Under such strong regularization, the 5th column
of functions is rank zero and, thus, identically zero.  The remaining
columns have rank one; the estimated fitted values are scalar
multiples of one another.
%% Comparison of the second row to the first row shows many
%% scalings that are less than one; hence, many of the functions in the
%% second row appear to have less extreme slopes than those in the first
%% row.  
We also see that scalings can be different for different columns.  The
third plot in the third row shows a slightly negative slope,
indicating a negative scaling for this particular estimate.  The
remaining functions in this row are oriented similarly to the other
rows, indicating the same, positive scaling.  This property
characterizes the difference between penalties 1 and 2; in an
application of penalty 2, the scalings would have been the same across
all functions in a given row.

Next, we illustrate a higher-rank solution for penalty 2.  Choosing
the regularization parameter $\lambda$ by ten-fold cross-validation
gives a fit of rank 5, considerably lower than 27, the maximum
possible rank.  Figure~\ref{f01} shows a selection of three
coordinates of the fitted functions.  Under rank five, each row of
functions is a linear combination of up to five other, linearly
independent rows.  We remark that the use of cross-validation
generally produces somewhat more complex models than is necessary to
capture an underlying low-rank data-generating mechanism.  Hence, if
the causal relationships for these data were indeed additive and low
rank, then the true low rank might well be smaller than five.

\subsection{Biochemistry Example}

\begin{figure}[t]
\begin{center}
\begin{tabular}{cc}
\quad Penalty 2, $\lambda=1$ & \quad No regularization \\[-25pt]
\hskip-20pt
\includegraphics[width=.6\textwidth,angle=-90]{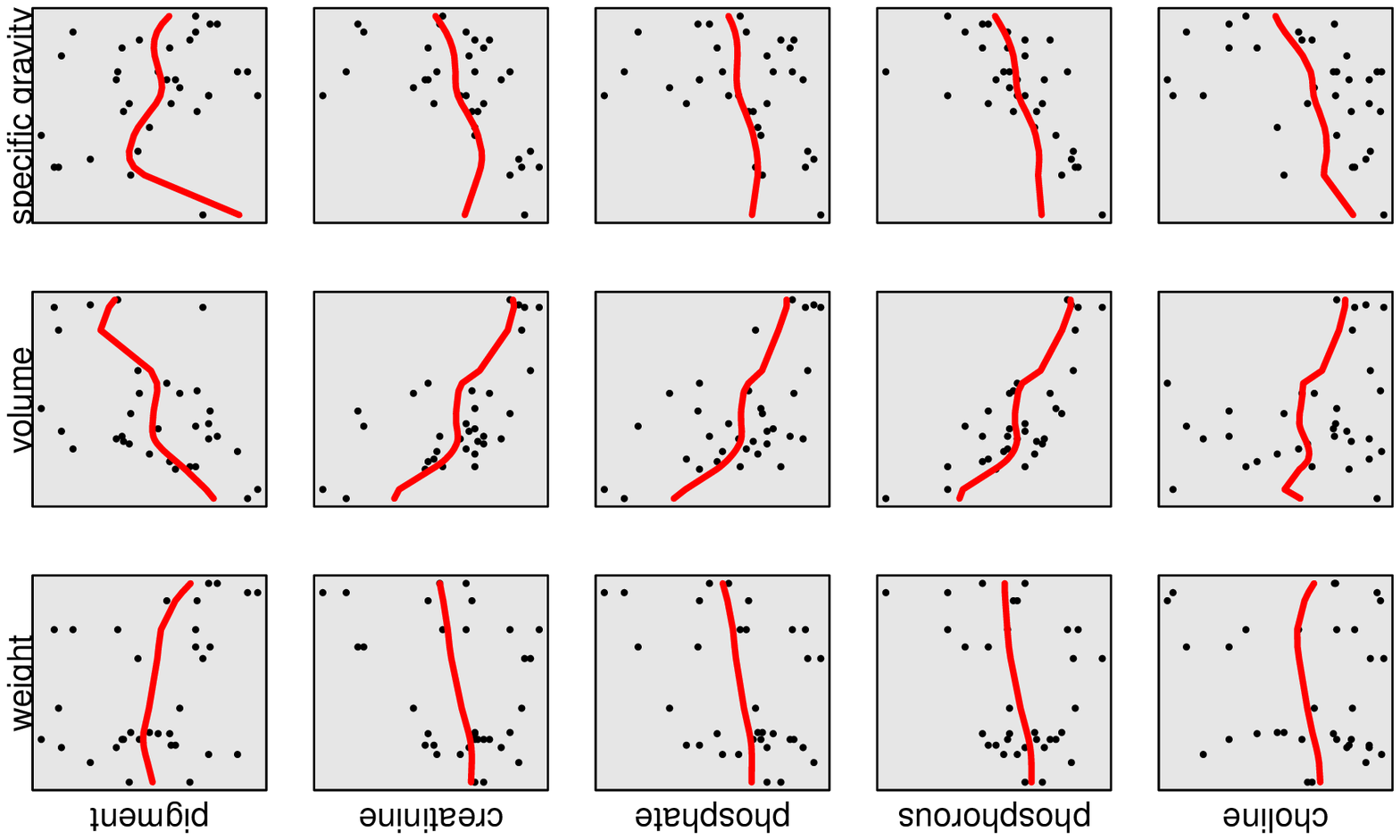}
&
\includegraphics[width=.6\textwidth,angle=-90]{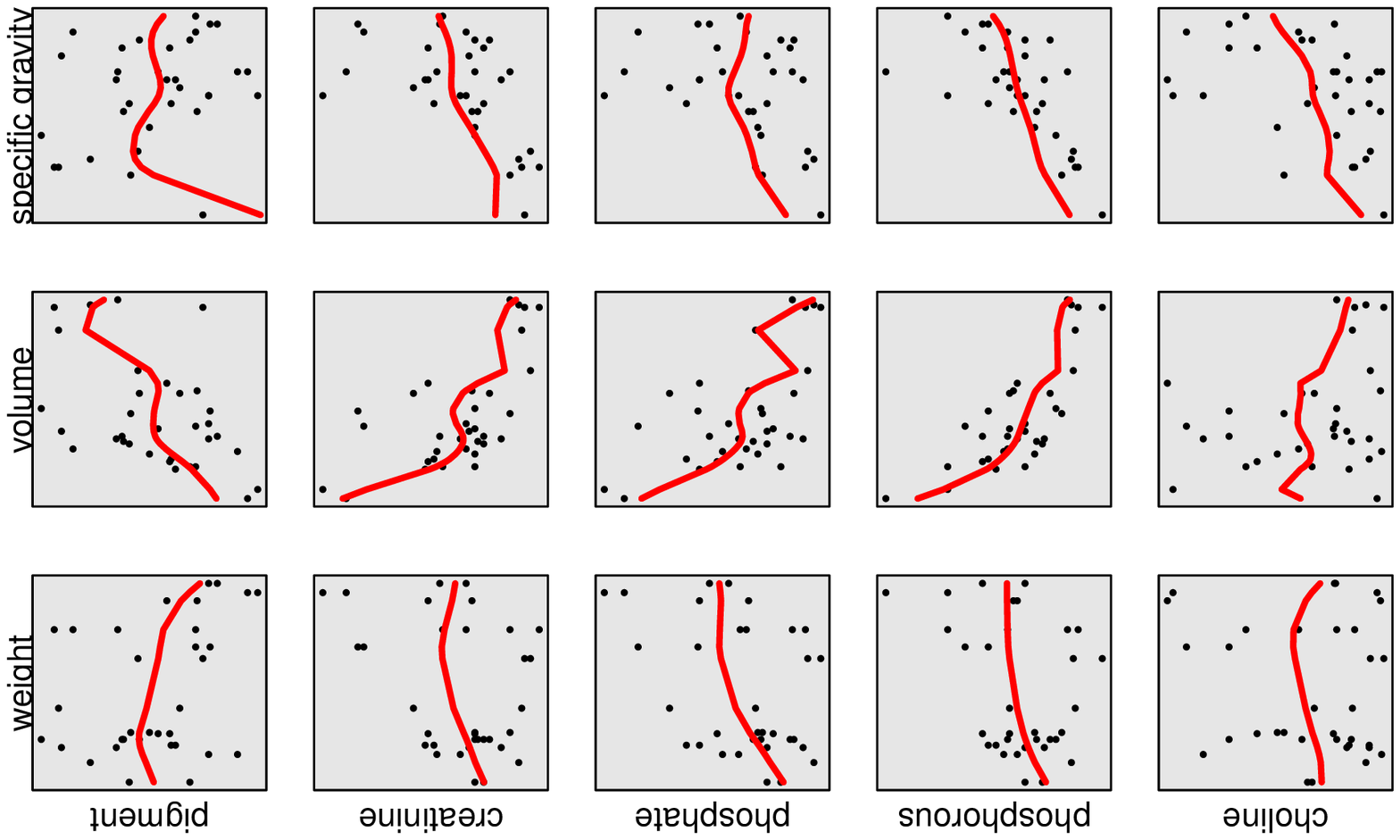}
\\[-10pt]
\end{tabular}
\end{center}
\caption{Fits on the biochemistry data with penalty 2 (left),
  with $\lambda=1$. The solution
  has rank three, with regression functions for 
  creatinine, phosphate and phosphorous identical up to scaling.  The
  fits on the right have no regularization, and are full
  rank.\label{biochem}}
\vskip-5pt
\end{figure}

Here we analyze the same biochemical data of Smith et
al.~(1962) that was used by \cite{Yuan:07}.  The data contain chemical measurements for
33 individual samples of men's urine
specimens.  The $q=5$ response variables are pigment creatinine, and
the concentrations (in mg/ml) of phosphate,
phosphorous, creatinine and choline.  The $p=3$ covariates are
the weight of the subject,  and volume and specific gravity of the
specimen.  \cite{Yuan:07} form a linear model where the coefficients
in a spline basis are regularized.  Their plots suggest some boundary
effects, due to the choice of basis.  Here we use our backfitting
algorithm with local linear smoothing.  No explicit basis is used.
Figure~\ref{biochem} shows the result of
using regularization $\lambda \trnnorm{M_{1:p}}$ with
$\lambda=1$ (penalty 2), and bandwidth $h=.3$ for all variables.
The regularized solution has rank 3, where
the response variables for phosphorous, phosphate, and creatinine
concentrations are scaled versions of each other.
The plots on the right show fits with no regularization ($\lambda=0$).

\section{Summary}

This paper introduced two penalties that induce reduced rank fits in
multivariate additive nonparametric regression.  Under linearity, the
penalties specialize to group lasso and nuclear
norm penalties for classical reduced rank regression.
Examining the subdifferentials of each of these penalties, we
developed backfitting algorithms for the two resulting optimization
problems that are based on soft-thresholding of singular values of
smoothed residual matrices. The algorithms were demonstrated on a gene expression data set
and a biochemical data set that have low-rank structure. We also provided a
persistence analysis that shows error tending to zero under a scaling
assumption on the sample size $n$ and the dimensions $q$ and~$p$.
% of
%the regression problem.

%This paper introduced two penalties that induce reduced rank fits in
%multivariate additive nonparametric regression.  The first penalty is
%separable over the different covariates.  It uses the nuclear norm to
%produce fits that exhibit linear relationships among the functions
%that relate an individual predictor to the different responses.  This
%penalty specializes to the group lasso for multi-task learning when
%the regression functions are linear.  The second nuclear norm penalty,
%on the other hand, leads to linear relationships that hold
%simultaneously across all the additive components of the regression
%functions.  Under linearity, this penalty specializes to the nuclear
%norm penalty for classical reduced rank regression.
%
%Examining the subdifferentials of each of these penalties, we
%developed backfitting algorithms for the two resulting optimization
%problems that are based on soft-thresholding of singular values. We
%then applied these algorithms to a gene expression data set
%constructed to have a naturally low-rank structure. We also provided a
%persistence analysis that shows error tending to zero under a scaling
%assumption on the sample size $n$ and the dimensions $q$ and $p$ of
%the regression problem.

\section*{Acknowledgements}
Research supported in part by NSF grants IIS-1116730, DMS-0746265, and DMS-1203762,
AFOSR grant FA9550-09-1-0373, ONR grant
N000141210762, and an Alfred P.~Sloan Fellowship.

%\clearpage
\bibliographystyle{imsart-nameyear}
\bibliography{local}

\end{document}